\documentclass{bmvc2k}

\title{Cost-Sensitive Learning for Long-Tailed Temporal Action Segmentation}

\addauthor{Zhanzhong Pang}{pang@comp.nus.edu.sg}{1}
\addauthor{Fadime Sener}{famesener@meta.com}{2}
\addauthor{Shrinivas Ramasubramanian}{shrinivr@cs.cmu.edu}{3}
\addauthor{Angela Yao}{ayao@comp.nus.edu.sg}{1}
\addinstitution{
 National University of Singapore,\\
 Singapore
}
\addinstitution{
 Meta Reality Labs\\
}
\addinstitution{
 Carnegie Mellon University,\\
 Pittsburgh, USA
}

\runninghead{Cost-Sensitive Learning for LT-TAS}{}

\def\eg{\emph{e.g}\bmvaOneDot}

\usepackage{multirow}
\usepackage{bbold}
\usepackage{xcolor}    
\usepackage{graphicx}
\usepackage{algorithm,multicol}
\usepackage[noend]{algpseudocode}
\usepackage{amsmath}
\usepackage{amsthm}
\usepackage{pifont}

\newtheorem{proposition}{Proposition}
\newcommand{\cmark}{\ding{51}}%
\newcommand{\xmark}{\ding{55}}%
\usepackage{float}
\usepackage{mathtools}
\usepackage{bm}
\usepackage{makecell}
\usepackage{soul}
\usepackage{wrapfig}
\usepackage{adjustbox}
\usepackage{verbatim}

\newcommand{\ie}{\textit{i}.\textit{e}., }

\begin{document}

\maketitle

\begin{abstract}
 
Temporal action segmentation in untrimmed procedural videos aims to densely label frames into action classes. These videos inherently exhibit long-tailed distributions, where actions vary widely in frequency and duration. In temporal action segmentation approaches, we identified a bi-level learning bias. This bias encompasses (1) a class-level bias, stemming from class imbalance favoring head classes, and (2) a transition-level bias arising from variations in transitions, prioritizing commonly observed transitions.
As a remedy, we introduce a constrained optimization problem to alleviate both biases. We define learning states for action classes and their associated transitions and integrate them into the optimization process. We propose a novel cost-sensitive loss function  formulated as a weighted cross-entropy loss, with weights adaptively adjusted based on the learning state of actions and their transitions. Experiments on three challenging temporal segmentation benchmarks and various frameworks demonstrate the effectiveness of our approach, resulting in significant improvements in both per-class frame-wise and segment-wise performance. Code is availabel at \url{https://github.com/pangzhan27/CSL_LT-TAS}.

\end{abstract}

\section{Introduction}
\label{sec:intro}

Temporal action segmentation identifies actions in untrimmed procedural video sequences. These sequences often exhibit a long-tail distribution as shown in Fig.~\ref{fig:teaser}~(a) with tail actions that occur less frequently or have shorter durations. Despite this, state-of-the-art methods often overlook the long-tail, failing to recognize tail actions. For example, AsFormer~\citep{yi2021asformer} and DiffAct~\cite{liu2023diffusion} exhibit zero accuracy on 5 and 4 out of 48 actions on Breakfast (see Fig.~\ref{fig:teaser}~(a) and Supplementary).
The long-tail issue in action segmentation remains unexplored~\citep{ding2022temporal,farha2019ms,yi2021asformer,singhania2021coarse,gao2021global2local} due to the widespread use of global evaluation metrics across all samples which obscure the poor performance on tail actions.

Long tail learning on videos has predominantly been explored in action recognition~\cite{zhang2021videolt,perrett2023use}. Action recognition~\cite{sener2020temporal,lin2019tsm,fan2021multiscale} aims at classifying trimmed video clips as a whole, while temporal action segmentation focuses on frame-wise classification of untrimmed videos, necessitating the modeling of temporal dynamics and action transitions for precise segmentation. Conventional solutions to long-tail learning focus on reducing the class imbalance via loss re-weighting~\citep{cui2019class, lin2017focal}, logit adjustment~\citep{wang2021seesaw, menon2020long}, and post-hoc adjustment~\citep{kang2019decoupling, menon2020long}. These approaches operate under a class-independent assumption, overlooking temporal dependencies and dynamics in temporal action segmentation, thus leading to inaccurate segments and transitions. Consequently, striking a balance between improving segmentation accuracy and minimizing adverse impacts on learned temporal dynamics poses a significant challenge.
 
\begin{figure}
\begin{tabular}{c}
\includegraphics[width=12.7cm]{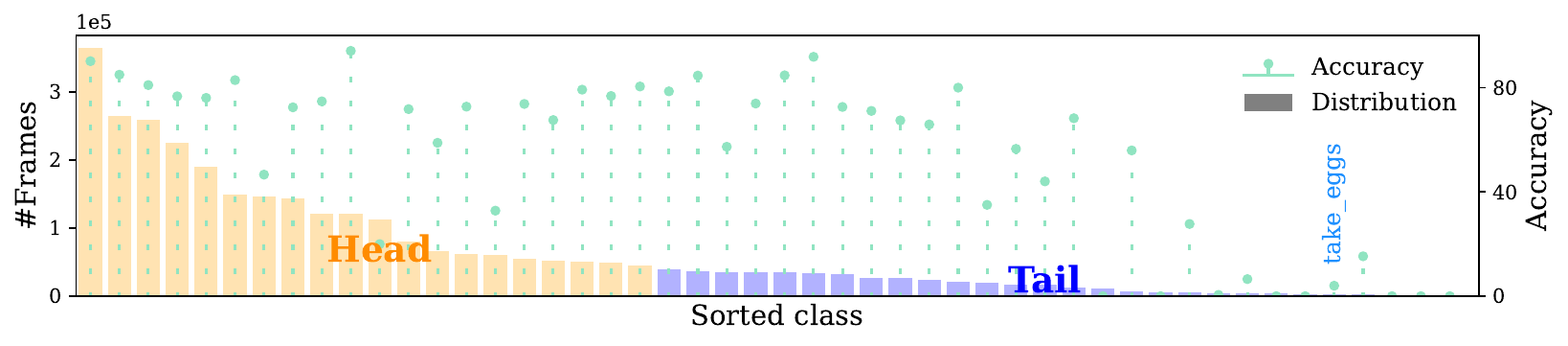}\\
(a) Data \& Class-wise accuracy distribution \\
\includegraphics[width=12.5cm]{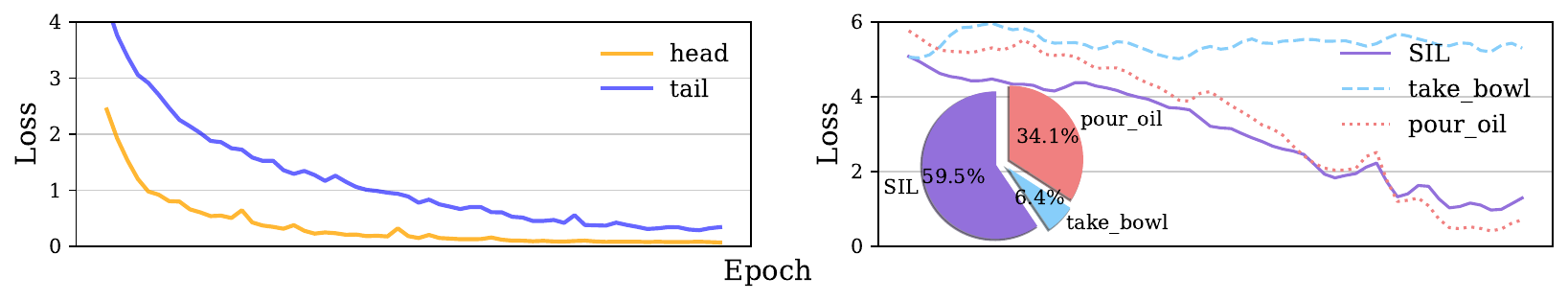} \\
(b) The bi-level learning bias
\end{tabular}
\caption{ 
(a) Long-tail action distribution on Breakfast~\citep{kuehne2014language}. The long-tail distribution results in low accuracy on tail actions with AsFormer~\citep{yi2021asformer}. (b) Left: Head-tail loss curve shows slow convergence rate on tail actions, demonstrating the class-level learning bias. Right: Action \emph{`take\_eggs'} from tail shows skewed transition distribution(pie chart), \ie different transitions from \{ \emph{`SIL'}, \emph{`pour\_oil'}, \emph{`take\_bowl'}\} to \emph{`take\_eggs'}, and transition learning bias(loss curve, where common transition from \emph{`pour\_oil'} are better learned than \emph{`take\_bowl'})}
\label{fig:teaser}
\end{figure}

Our paper addresses the long tail issue in temporal action segmentation, bridging the research gap of long-tailed learning for untrimmed videos. 
Empirically, we observe a bi-level biased learning process attributed to the long-tail problem. 

\begin{itemize}

\item Class imbalance leads to a \textbf{class-level learning bias}, which prioritizes learning head over tail actions, leading to different class convergence rates~(Fig.~\ref{fig:teaser}~(b) Left). However, unlike the typical over-fitting to tail observed in long-tailed image classification and segmentation~\citep{buda2018systematic,hsieh2021droploss,wang2021adaptive,samuel2021generalized,samuel2021distributional}, we observe under-learnt tail actions in temporal segmentation. This is because the learning of tail is suppressed due to the temporal continuity of frame representation, see~Fig.~\ref{fig:feature}. Distinctly separating two consecutive actions, one being head and the other tail, is challenging as they share similar frame representations, especially at segment boundaries. This similarity in representation hinders independent learning of tail actions without adversely affecting head actions.

\item Variations in action transitions introduce a \textbf{transition-level learning bias}. In Fig.~\ref{fig:teaser}~(b) Right, for action \emph{'take\_eggs'}, the transition distribution from  \emph{'pour\_oil'} or \emph{'take\_bowl'} to \emph{'take\_eggs'} is skewed. We observe a higher frequency of \emph{'take\_eggs'}  preceded by \emph{'pour\_oil'}. Such frequent transitions, \eg from \emph{'pour\_oil'}, tend to form stronger associations, resulting in learning gaps across transitions. For instance, \emph{'take\_eggs'} is more easily detected when preceded by \emph{'pour\_oil'} compared to \emph{'take\_bowl}. 
 
\end{itemize}
 
To address these biases, we propose utilizing the class-wise accuracy to evaluate action learning state and transition-wise accuracy for transition learning state. These evaluations determine if an action or its transition is over- or under-learned by comparing them to their respective average accuracy. We design a constrained optimization problem targeting a balanced accuracy to reduce class-level bias. Constraints on temporal transition learning are also imposed to address transition-level bias. Incorporating these constraints into a deep learning framework is nontrivial. To tackle this, we reframe the optimization as a Lagrangian min-max problem, which can be optimized by minimizing a surrogate cost-sensitive loss function. Our new loss function, a weighted cross-entropy formulation, adjusts weights adaptively based on the learning state of actions and their transitions.

Our contributions can be summarized as: (1) identifying the bi-level learning bias and the under-learned tail classes in temporal action segmentation, which differs from the common over-fitting trends observed in other tasks, (2) proposing a cost-sensitive loss that addresses these biases via a constraint optimization formulation, (3) conducting extensive evaluations on different backbones and datasets, showcasing notable performance improvements. 

\section{Related works}
\label{sec:related}

\textbf{Temporal Action Segmentation} employs various architectures such as temporal convolutional networks (TCN)~\citep{lea2017temporal,li2020ms,lei2018temporal,singhania2021coarse,farha2019ms}, transformers~\citep{yi2021asformer,behrmann2022unified}, and diffusion models~\citep{liu2023diffusion}. These architectures expand the temporal receptive field~\cite{singhania2021coarse,farha2019ms} and aggregate temporal dynamics~\cite{yi2021asformer,behrmann2022unified}, facilitating information exchange across frames. To address the over-segmentation in such backbones, several approaches like boundary smoothing~\citep{ishikawa2021alleviating, wang2020boundary} and refinement~\citep{behrmann2022unified} has been proposed. Moreover, to incorporate temporal constraints in these backbones, differentiable temporal logic~\citep{xu2022don} and activity grammar~\citep{gong2024activity} are utilized.
 
\noindent \textbf{Long-Tail Learning} involves various techniques. Re-sampling methods either undersample the head~\citep{buda2018systematic,shen2016relay,byrd2019effect} or oversample the tail~\citep{he2009learning,drummond2003c4}. Re-weighting assigns different weights to classes~\citep{wang2017learning,huang2016learning,cui2019class} or samples~\citep{ren2018learning,lin2017focal}. Logit adjustment modifies margins based on class priors~\citep{menon2020long,cao2019learning} or compensation terms~\citep{tan2020equalization,wang2021seesaw,zhao2022adaptive}. Post-hoc adjustment includes normalizing classification weights~\citep{kang2019decoupling,zhang2019balance,kim2020adjusting} or modifying thresholds~\citep{collell2016reviving,king2001logistic}. These methods have been extended to object detection/segmentation~\citep{li2020overcoming,tan2020equalization} and video classification~\cite{zhang2021videolt,perrett2023use}. 

Temporal action segmentation differs from these tasks due to temporal correlations between frames and segments. The long-tail issue in this domain remains unexplored. Our work addresses the long-tail in temporal action segmentation, aiming to tackle learning biases while accounting for temporal dynamics.

\section{Method}
\label{sec:method}
In temporal action segmentation, a classifier $f$ maps a video sequence $X \in \mathbb{R}^{D \times T}$ represented with pre-computed features~\cite{carreira2017} to a sequence of actions $Y \in [L]^{T}$. Here, $D$ is the feature dimension, $T$ indicates the number of frames, and $L$ represents the number of classes. Classifier $f$ is typically a neural network backbone such as MSTCN~\cite{farha2019ms} or AsFormer~\cite{yi2021asformer}, where segmentation is usually framed as frame-wise classification.

Our paper presents a cost-sensitive learning framework to tackle the long-tail issue in action segmentation. We evaluate the learning states of both actions (class-level) and action transitions (transition-level) using a transition-based confusion tensor in Section \ref{sec:method1}. We then formulate a learning-aware constrained optimization problem that is transformed it into a new cost-sensitive loss~\citep{narasimhan2021training,rangwani2022cost,he2022relieving} in Section \ref{sec:method2}. We provide details on training with the cost-sensitive loss and a new post-processing technique for inference in Section \ref{sec:method3}.

\begin{figure*}[tb]
\centering
\begin{minipage}[b]{.49\textwidth}
 \centering
 \includegraphics[scale=0.24]{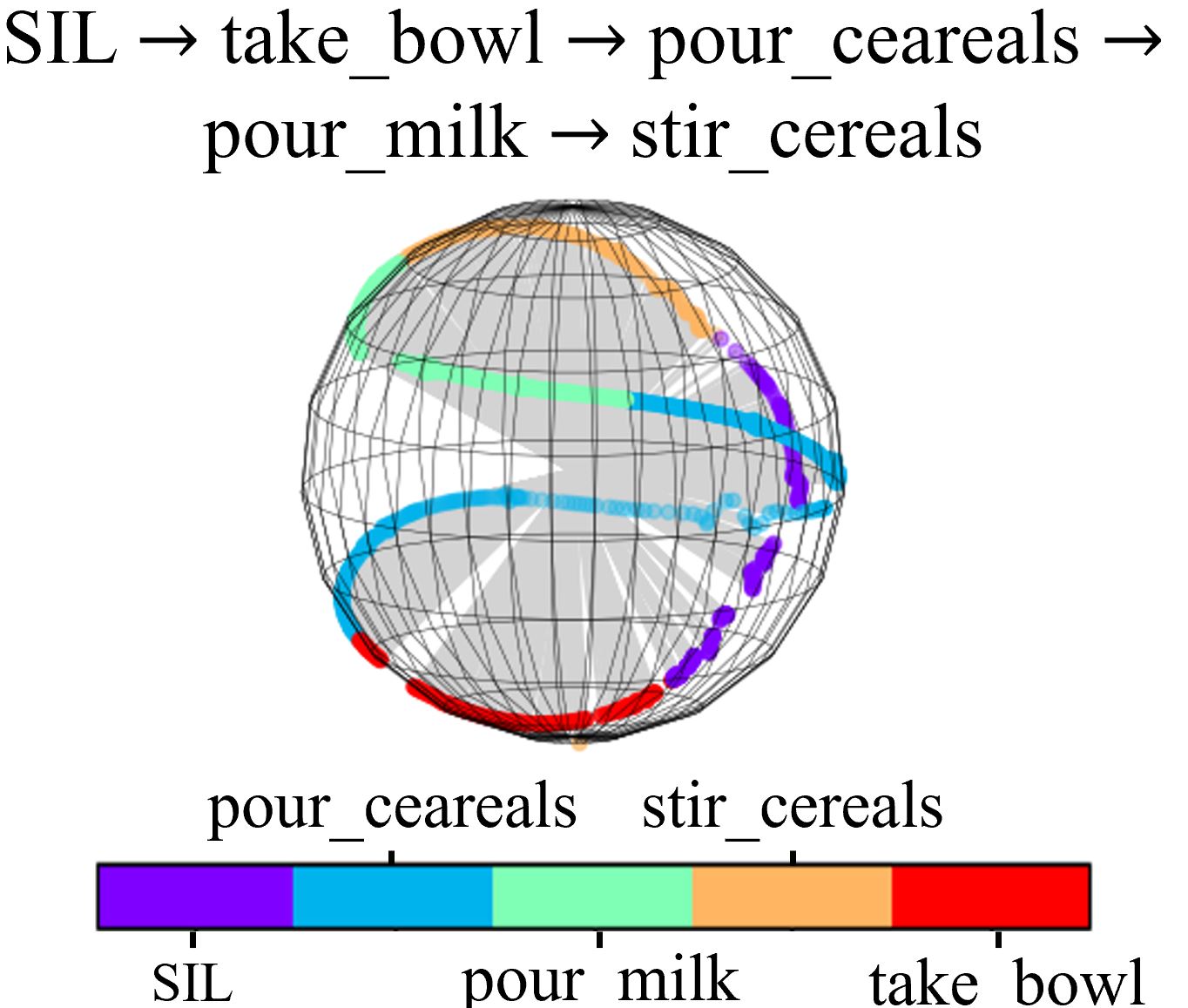}
 \caption{The t-SNE of the frame-wise representations for a video of making cereal exhibits a strong temporal continuity.}
 \label{fig:feature}
\end{minipage}
\
\begin{minipage}[b]{.49\textwidth}
 \centering
 \includegraphics[scale=0.33]{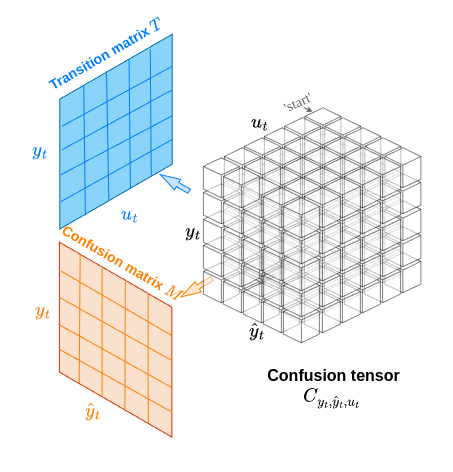}
 \caption{Transition-based confusion tensor.}
 \label{fig:tensor}
\end{minipage}%
\end{figure*}

\subsection{Class- \& Transition-level Learning States}
\label{sec:method1}
A video sequence with $N$ actions and $T$ frames can be labelled either on a frame-wise basis, $Y = \{y_t\}_{t=1}^T$, with frame index $t$, or on a segment-wise basis, $Y = \{a_n\}_{n=1}^N$, with segment index $n$. The action segment $a_n = (s_n, e_n, l_n)$ represents a segment with start time $s_n$, end time $e_n$, and label $l_n$. For a timestamp $t \in [s_n, e_n]$, its frame label $y_t = l_n$.

To reduce the class-level learning bias, it is important to guarantee all classes are equally learned, namely achieving a uniform learning across classes. Similarly, all transitions should be learned equally to reduce transition-level learning bias. We assess the learning state for an action or a transition during training using its corresponding accuracy on training set. This can be calculated with a transition-based confusion tensor. Given a classifier $f$, the ${ijk}^{\text{th}}$ entry of its confusion tensor $C$ is defined as
\begin{equation}
 C_{i,j,k}[f] = \mathbb{E}_{(X,Y)}[\mathbb{1}(y_t=i, \hat{y}_t=j, u_t =k)] 
\end{equation}
where for the $t^{\text{th}}$ frame, $y_t$ is the ground truth, $\hat{y}_t$ is the prediction from $f$, $u_t$ represents the previous action of the current frame, namely, $\forall t \in [s_n, e_n], y_t = l_n, u_t = l_{n-1}$, with $l_{-1} = \text{'start'}$. 

Based on the confusion tensor $C$, we can derive the corresponding confusion matrix $M$ and transition matrix $T$ shown in Fig.~\ref{fig:tensor} as 
\begin{equation}
 M_{i,j}[f] = \sum_k C_{i,j,k}[f], \quad
 T_{i,k} = \sum_j C_{i,j,k}[f]
\end{equation}
Note the transition matrix $T$ remains constant regardless of the classifier $f$. It can be derived solely from the training dataset and remains consistent across variants of models or initialization. For simplicity, we omit the dependency of $T$ on $f$.

Then, for an action class $i$ with a transition from another class $k$, $k \neq i$, the learning states for class $i$ and the transition associated $i$ and $k$ are formulated as the corresponding accuracy: 
\begin{align}
 Acc_i[f] = \frac{M_{i,i}[f]}{\pi_i} = \frac{\sum_k C_{i,i, k}[f]}{\pi_i}, \quad
 Tacc_{k\rightarrow i}[f] = \frac{C_{i,i,k}[f]}{T_{i,k}},
\label{eq:acc}
\end{align}
where $\pi_i$ is the class prior $p(y=i) = \sum_k T_{i,k}$.

\subsection{Cost-sensitive Learning with Constraint Optimization}
\label{sec:method2}
We propose a new learning objective to mitigate class-level learning bias. The objective is to maximize per-class accuracy $\max\limits_{f} \sum_i Acc_i[f]$, ensuring equal attention to all classes. 

To further reduce the biased transition learning, we define transition constraints to reduce the learning variance. Let $V_T$ denote the set of valid action transitions observed in the training set; specifically, $V_T = \{(k \rightarrow i), T_{i,k} > 0 \}$. One way to regularize the transition learning is to penalize under-learned transitions:
\begin{equation}
\forall (k \rightarrow i) \in V_T, \ \ Tacc_{k \rightarrow i}[f] \ge \epsilon \overline{Tacc}, \quad \text{where} \quad \overline{Tacc} = \frac{1}{|V_T|} \sum_{(k \rightarrow i) \in V_T} Tacc_{k \rightarrow i}[f],
\end{equation}
where $\epsilon$ is a tolerance hyperparameter which set to 0.9 in our implementation. $\overline{Tacc}$ is the average accuracy over all transitions. To simplify the problem, we detach $\overline{Tacc}$ from the classifier $f$, \ie $\overline{Tacc}$ is not considered as a function of $f$. We set it as a parameter and update it every epoch during training.

Then, the objective on per-class accuracy and above constraints combine to the problem: 
\begin{equation}
 \max\limits_{f} \sum_{i,k} \frac{C_{i,i,k}[f]}{\pi_i} \quad \quad \text{s.t.} \ \forall (k \rightarrow i) \in V_T, \ \ \frac{C_{i,i,k}[f]}{T_{i,k}} \ge \epsilon \overline{Tacc}.
 \label{eq: op0}
\end{equation}

Optimizing Eq. (\ref{eq: op0}) with constraints is challenging. A common strategy is to relax the constraints and reformulate the objective as a Lagrangian. 
Eq.~(\ref{eq: op0}) can be reformulated as an equivalent Lagrangian min-max problem $\mathcal{L}(f, \lambda)$ by introducing Lagrange multipliers $\lambda$:
\begin{equation}
 \max\limits_{f} \min\limits_{\lambda \in \mathbb{R}_+} \sum_{i, k} \frac{C_{i,i,k}[f]}{\pi_i} + \sum_{i, k} \mathbb{1}(T_{i,k} > 0) \lambda_{i,k} \left( \frac{C_{i,i,k}[f]}{T_{i,k}} - \epsilon \overline{Tacc} \right) \frac{T_{i,k}}{\pi_i},
\label{eq: op}
\end{equation}
\noindent where a constant term $\frac{T_{i,k}}{\pi_i}$ for a given transition pair $k \rightarrow i$ is multiplied by each transition constraint to balance the magnitudes between the objective and constraints. 

The Lagrangian is solved iteratively by maximizing $f$ while fixing the multipliers $\lambda$ and minimizing $\lambda$ while keeping $f$ fixed. In practice, instead of the full optimization at each max$\slash$min iteration, we only take a few update steps for gradient descent to update the classifier $f$ and for projected gradient descent to update the Lagrange multiplier. The detailed training algorithm can be in Algorithm~\ref{alg:1}. There are theoretical guarantees on convergence for the learned classifier~\citep{chen2017robust, cotter2019optimization, narasimhan2021training}. 

\begin{algorithm}
	\caption{Optimizing Per class  Accuracy with Transitional Constraints}
	\label{algo:0}
    \hspace*{\algorithmicindent} \textbf{Input:} \ Training set $\mathcal{D}$, Class prior $\pi \in \mathbb{R}_+^{L}$ and transition prior $T \in \mathbb{R}_+^{L \times (L+1)}$ derived from $\mathcal{D}$, Learning rate for multiplier $\gamma \in \mathbb{R}_+$, Cost-sensitive loss function $l$, Lagrangian objective $\mathcal{L}$ \\
    \hspace*{\algorithmicindent} \textbf{Initialize:} Classifier $f$, Multiplier $\lambda \in \mathbb{R}_+^{L \times (L+1)}$
	\begin{algorithmic}[1]
        \For{ epoch $ \ l \leftarrow 0, \dots, N $} 
            \State \textcolor{magenta}{\emph{// Update $G$}}
            \State Calculate the gain tensor $G$ based on $\pi$, $T$, and $\lambda$
            \State \textcolor{magenta}{\emph{// Update $f$}}
		\State $ f^{l+1} \in \arg\min_{f} \frac{1}{|\mathcal{D}|} \sum_{(X,Y) \in \mathcal{D}} l(y_t, \hat{y}_t, G) $  \emph{ // $Y = \{y_t\}$}
		\State \textcolor{magenta}{\emph{// Update $\lambda$}}
        \State $C_{i,j,k}[f^{l+1}] = \frac{1}{|\mathcal{D}|} \sum_{(X,Y)\in \mathcal{D}} \mathbb{1}(y_t=i, \hat{y}_t=j, u_t=k)$ \emph{ // calculate confusion matrix}
        \State Calculate $\overline{Tacc}$ based on $T$ and $C[f^{l+1}]$
		\State $\lambda_{i,k}^{l+1} = \max \{\lambda_{i,k}^{l} - \gamma \nabla_{\lambda_{i,k}}\mathcal{L}, 0\}$ {\emph{// gradients are calculated based on $\overline{Tacc}$}}
        \EndFor
	\end{algorithmic}  
 \label{alg:1}
\end{algorithm} 

\noindent \textbf{Step 1. Maximizing the Lagrangian $\mathcal{L}(f)$ with fixed $\lambda$ } leads to the following objective:
\begin{equation}
 \max\limits_{f} \sum_{i, k} G_{i,i,k} C_{i,i,k}[f] + \text{constant},
\label{eq:max}
\end{equation}
\noindent where $G$ is a \emph{gain tensor} representing the gain of the correct classification and transition, while `constant' absorbs terms not depending on $f$. Given a transition from action $k$, the slice $G_{:,:,k}$ from the gain tensor $G$ is a diagonal matrix, and $G_{i,i,k} = (1 + \mathbb{1}(T_{i,k} > 0) \lambda_{i,k}) / \pi_i $. 
Optimizing Eq. (\ref{eq:max}) is equivalent to minimizing a re-weighted loss in Eq. (\ref{eq:rw}), which is proven to be calibrated for this diagonal gain matrix~\citep{patrini2017making, narasimhan2021training}. See Supplementary for the proof. 
\begin{equation}
l_{CE}(y_t, u_t, X) = -G_{y_t, y_t, u_t} \log(p(y_t \mid X))
\label{eq:rw}
\end{equation}

The formulation in Eq.~(\ref{eq:max}) represents a more generalized form of cost-sensitive learning~\citep{narasimhan2021training}.
Standard cost-sensitive learning is typically formulated as naive reweighting based on the class frequency, treating each class independently. In contrast, our proposed reweighting factor $G_{i,i,k}$ considers class inter-dependencies, incorporating an extra term that models the transitions as shown in Eq. (\ref{eq: core}). This extra term allows adaptive adjustment of the reweighting factor for a given action based on its current transition learning state. 
\begin{equation}
 G_{i,i,k} = {\underbrace{ \frac{1}{\pi_i}}_{\text{action prior}}} + {\underbrace{ \frac{\lambda_{i,k}\mathbb{1}(T_{i,k} > 0)}{\pi_i}}_{\text{transition learning state}}} 
 \label{eq: core}
\end{equation}

\noindent \textbf{Step 2. Minimizing the Lagrangian $\mathcal{L}(\lambda)$} is done by projected gradient descent. The gradients of the Lagrangian objective $\mathcal{L}$ in Eq. (\ref{eq: op}) with respect to $\lambda$ 
is estimated as $\nabla_{\lambda}\mathcal{L}$. The multipliers $\lambda$ are updated with gradient descent and projected back to $\mathbb{R}_+$ as
\begin{equation}
 \lambda^{(l+1)} = \max\{0, \lambda^{(l)} - \gamma \nabla_{\lambda}\mathcal{L} \},
\end{equation}
where $\gamma$ is the step size for updating the multipliers, $l$ is the iteration index.

\subsection{Training and Inference}
\label{sec:method3}
Empirically, over-emphasizing the per-class performance will hurt the global performance. To achieve a better trade-off, we introduce a hyper-parameter, $\tau$, to modify the gain tenser as $\tilde{G}_{i,j,k} = G_{i,j,k}^{\tau}$ for the frame-wise loss in Eq. (\ref{eq:rw}) as 
\begin{equation}
l_{CE}(y_t, u_t, X) = -G_{y_t, y_t, u_t}^{\tau} \log(p(y_t \mid X))
\label{eq:rw1}
\end{equation}
A small $\tau$ will smooth the weights of the loss function, favoring global performance. Conversely, a large $\tau$ emphasizes reducing biased learning and enhancing per-class performance. We estimate the confusion tensor $C$ every epoch on training set to update Lagrangian multipliers $\lambda$ and the gain tensor $G$. Importantly, this modification only affects training, leaving the inference stage unchanged.

Similar to \cite{kang2019decoupling}, we identify that the final classifier is biased. Inspired by Nearest Class Mean(NCM)~\cite{snell2017prototypical}, we propose a new post-processing technique to further mitigate the long-tail impact.
Specifically, instead of relying on the classifier, we make predictions using frame representations based on NCM, which involves computing mean representations for each class and performing nearest neighbor search using Euclidean distance. Applying frame-level NCM, however, disregard the temporal continuity and lead to over-segmentation. We propose Segment Nearest Class Mean (S-NCM) to address this. As in Eq.~(\ref{eq:infer}), we first leverage the classifier's predictions $\hat{y}$ to detect segment boundaries $b$ and then utilize frame-wise NCM predictions $\hat{v}$ for labelling each segment through major voting, namely the frames in each segment share the same prediction $\tilde{y}$. The segment label is determined by selecting the most frequent prediction (mode) from all frame-wise NCM predictions $\hat{v}$.
\begin{equation}
 \tilde{y}_{b_i : b_{i+1}} = \text{mod}(\hat{v}_{b_i}, \hat{v}_{b_i +1}, \cdots, \hat{v}_{b_{i+1}})
 \quad \text{where} \quad
 b = \{t, \ \ \text{if} \ \ \hat{y}_t != \hat{y}_{t+1}, \ \forall t \in [1, T-1]\}
 \label{eq:infer}
\end{equation}

\section{Experiments}
\label{sec:experiments}
\subsection{Dataset, Implementation, and Evaluation}
\textbf{Dataset. }
We evaluate our framework on three benchmarks: Breakfast Actions~\citep{kuehne2014language}, 50Salads~\cite{stein2013combining} and the recently released Assembly101~\citep{sener2022}. Breakfast comprises 1712 videos for breakfast preparation, featuring 48 action classes with an average duration of 2.3 minutes. Assembly101 has a collection of 4321 videos focused on assembling and disassembling toys, with an average length of 7.1 minutes and 202 coarse action classes.  50Salads contains 50 videos of making mixed salads, involving 19 actions.We split the actions in these datasets into Head and Tail groups based on the class frequency as in Table \ref{tab:cls_split} and evaluate the performance of different methods on each group. 

\begin{table}[htb]
  \caption{ Head$\slash$Tail class split criterion.}
  \label{tab:data}
  \centering
  \begin{tabular}{c|c|c|c|c}
    \hline
    \multirow{2}{*}{\textbf{Dataset}}  & \multicolumn{2}{c|}{\textbf{Head}} & \multicolumn{2}{c}{\textbf{Tail}} \\
    \cline{2-5} & $\#$classes & $\#$frames  & $\#$classes & $\#$frames \\
    \hline
    Breakfast & 20 & $\ge$ 5$\times 10^4$ & 28 & $\le$ 5$\times 10^4$ \\
    50salads & 6 & $\ge$ 4$\times 10^4$ & 13 & $\le$ 4$\times 10^4$ \\
    Assembly101 & 31 & $\ge$ 1.8 $\times 10^5$ &  171 & $\le$ 1.8 $\times 10^5$  \\ 
    \hline
  \end{tabular}
  \label{tab:cls_split}
\end{table}

\noindent \textbf{Implementation details.} 
We consider three backbones: a temporal convolution model MSTCN~\citep{farha2019ms}, a transformer model ASFormer~\citep{yi2021asformer}, and a state of the art diffusion-based model DiffAct~\citep{liu2023diffusion}. All models are retrained using the released source codes based on I3D features~\citep{carreira2017} pre-trained on Kinetics. Results on Breakfast and 50salads are reported based on standard 4- and 5-fold splits respectively, while for Assembly101, we employ the provided train-val-test split and report test results. All long-tailed methods are trained with the same settings as the original baseline. 

\noindent \textbf{Evaluation metrics.} Three commonly used metrics~\citep{farha2019ms,singhania2021coarse,yi2021asformer,wang2020boundary} are: frame-wise accuracy (Acc.), segment-wise edit score (Edit), and F1 score with IoU thresholds of 0.10, 0.25 and 0.50 (F1@10/25/50). Conventionally, these metrics are tabulated globally over all the frames, obscuring the performance of tail actions.
To emphasize the performance of tail actions, we use balanced metrics commonly used in long-tailed works\citep{wang2021seesaw,tan2020equalization,kang2019decoupling}.
Specifically, we calculate the average of recall scores per class for frame-wise accuracy and use the per-class F1 score for the segment-wise evaluation. Under the long-tailed learning setting, we primarily report per-class performance in the main results. Please refer to the Supplementary for global metrics like Edit score, global accuracy, and global F1.

\subsection{Benchmark Results}
\begin{table}[t]
\caption{Per-class \& global result summary across datasets and backbones over 3 runs. The column 'G\_F1' represents the global F1 score with IOU threshold 0.25 over all samples.}
\centering
\resizebox{1.01\columnwidth}{!}{
\begin{tabular}{c|ccc|c|c|ccc|c|c|ccc|c|c}
\hline
\multirow{2}{*}{\centering{\textbf{Model}}} & \multicolumn{5}{c|}{\textbf{Breakfast}} & \multicolumn{5}{c|}{\textbf{50salads}} & \multicolumn{5}{c}{\textbf{Assembly101}} \\ 
\cline{2-16} & \multicolumn{4}{c|}{\textbf{Per class}} & \multirow{2}{*}{\textcolor{gray}{\textbf{G\_F1}}} & \multicolumn{4}{c|}{\textbf{Per class}} & \multirow{2}{*}{\textcolor{gray}{\textbf{G\_F1}}} & \multicolumn{4}{c|}{\textbf{Per class}} & \multirow{2}{*}{\textcolor{gray}{\textbf{G\_F1}}} \\
\cline{2-5} \cline{7-10} \cline{12-15} & \multicolumn{3}{c|}{F1@\{10,25,50\}} & Acc. & & \multicolumn{3}{c|}{F1@\{10,25,50\}} & Acc. & & \multicolumn{3}{c|}{F1@\{10,25,50\}} & Acc. & \\ 
\hline
\textbf{MSTCN}~\citep{farha2019ms} & 48.1 & 44.8 & 36.9 & 49.1 & \textcolor{gray}{57.9} & 78.8 & 76.4 & 67.6 & 75.6 & \textcolor{gray}{75.9} & 7.5 & 6.6 & 4.8 & 8.3 & \textcolor{gray}{27.2}\\
+ CB~\cite{cui2019class} & \textcolor{blue}{+0.9} & \textcolor{blue}{+0.7} & \textcolor{blue}{+0.3} & \textcolor{blue}{+0.6} & \textcolor{gray}{0.0} & \textcolor{red}{-0.6} & \textcolor{red}{-0.2} & \textcolor{red}{-0.8} & \textcolor{red}{-0.3} & \textcolor{gray}{-0.4} & \textcolor{blue}{+1.8} & \textcolor{blue}{+1.7} & \textcolor{blue}{+1.2} & \textcolor{blue}{+1.5} & \textcolor{gray}{-0.5}\\
+ LA~\cite{menon2020long} & \textcolor{blue}{+1.0} & \textcolor{blue}{+1.1} & \textcolor{blue}{+0.1} & \textcolor{blue}{+1.4} & \textcolor{gray}{0.0} & \textcolor{red}{-0.2} & \textcolor{red}{-0.7} & 0.0 & \textcolor{red}{-0.3} & \textcolor{gray}{-0.7} & \textcolor{blue}{+2.1} & \textcolor{blue}{+1.4} & \textcolor{blue}{+1.2} & \textcolor{blue}{+1.2} & \textcolor{gray}{-1.1}\\
+ Focal~\cite{lin2017focal} & \textcolor{blue}{+0.2} & \textcolor{red}{-0.3} & \textcolor{red}{-1.2} & \textcolor{red}{-0.5} & \textcolor{gray}{-0.4} & \textcolor{blue}{+0.6} & \textcolor{blue}{+0.5} & \textcolor{blue}{+1.0} & \textcolor{blue}{+0.4} & \textcolor{gray}{+0.2} & \textcolor{blue}{+1.9} & \textcolor{blue}{+1.6} & \textcolor{blue}{+0.5} & \textcolor{blue}{+1.4} & \textcolor{gray}{-0.2}\\
+ $\tau$-norm~\cite{kang2019decoupling} & \textcolor{red}{-1.1} & \textcolor{red}{-1.0} & \textcolor{red}{-1.0} & \textcolor{red}{-0.8} & \textcolor{gray}{-0.9} & \textcolor{red}{-0.6} & \textcolor{red}{-0.5} & \textcolor{red}{-0.2} & \textcolor{blue}{+0.2} & \textcolor{gray}{-0.6} & \textcolor{blue}{+0.1} & \textcolor{blue}{+0.2} & \textcolor{blue}{+0.1} & \textcolor{red}{-0.2} & \textcolor{gray}{+0.2}\\
+ ours(S-NCM) & \textcolor{blue}{\textbf{+8.1}} & \textcolor{blue}{\textbf{+8.1}} & \textcolor{blue}{\textbf{+5.7}} & \textcolor{blue}{\textbf{+3.7}} & \textcolor{gray}{\textbf{+6.1}} 
 & \textcolor{blue}{\textbf{+2.8}} & \textcolor{blue}{\textbf{+3.1}} & \textcolor{blue}{\textbf{+3.2}} & \textcolor{blue}{\textbf{+1.7}} & \textcolor{gray}{\textbf{+3.1}}
 & \textcolor{blue}{\textbf{+4.1}} & \textcolor{blue}{\textbf{+3.3}} & \textcolor{blue}{\textbf{+2.0}} & \textcolor{blue}{\textbf{+2.6}} & \textcolor{gray}{\textbf{+2.3}}\\
\hline
\hline
\textbf{ASFormer}~\citep{yi2021asformer} & 57.9 & 54.7 & 45.4 & 52.3 & \textcolor{gray}{69.9} & 85.1 & 82.6 & 75.3 & 81.5 & \textcolor{gray}{82.3} & 9.2 & 7.6 & 5.2 & 9.2 & \textcolor{gray}{30.4}\\
+ CB~\cite{cui2019class} & \textcolor{blue}{+0.8} & \textcolor{blue}{+1.0} & \textcolor{blue}{+0.9} & \textcolor{blue}{+0.8} & \textcolor{gray}{-0.2} & \textcolor{blue}{+0.2} & \textcolor{blue}{+1.0} & \textcolor{blue}{+0.9} & \textcolor{blue}{+0.2} & \textcolor{gray}{+0.9} & 0.0 & \textcolor{red}{-0.1} & 0.0 & \textcolor{blue}{+0.2} & \textcolor{gray}{-2.2}\\
+ LA~\cite{menon2020long} & \textcolor{blue}{+1.4} & \textcolor{blue}{+1.0} & \textcolor{blue}{+1.3} & \textcolor{blue}{+0.5} & \textcolor{gray}{-0.2} & \textcolor{blue}{+0.2} & \textcolor{blue}{+0.9} & \textcolor{blue}{+1.5} & \textcolor{blue}{+0.4} & \textcolor{gray}{+0.9} & \textcolor{red}{-0.1} & \textcolor{blue}{+0.4} & \textcolor{blue}{+0.1} & \textcolor{blue}{+0.3} & \textcolor{gray}{-1.9}\\
+ Focal~\cite{lin2017focal} & \textcolor{blue}{+1.0} & \textcolor{blue}{+1.2} & \textcolor{blue}{+0.4} & \textcolor{red}{-0.3} & \textcolor{gray}{\textbf{+0.5}} & \textcolor{blue}{+0.7} & \textcolor{blue}{+0.8} & \textcolor{blue}{+1.1} & \textcolor{red}{-0.3} & \textcolor{gray}{+1.2} & \textcolor{blue}{+1.5} & \textcolor{blue}{+2.1} & \textcolor{blue}{+1.1} & \textcolor{blue}{+1.7} & \textcolor{gray}{-0.1}\\
+ $\tau$-norm~\cite{kang2019decoupling} & 0.0 & \textcolor{blue}{+0.2} & \textcolor{blue}{+0.4} & \textcolor{blue}{+0.8} & \textcolor{gray}{-0.8} & \textcolor{red}{-0.1} & 0.0 & \textcolor{blue}{+0.1} & \textcolor{blue}{+0.1} & \textcolor{gray}{-0.1} & \textcolor{red}{-1.9} & \textcolor{red}{-2.1} & \textcolor{red}{-1.3} & \textcolor{red}{-1.0} & \textcolor{gray}{-7.7}\\
+ ours(S-NCM) & \textcolor{blue}{\textbf{+3.1}} & \textcolor{blue}{\textbf{+3.2}} & \textcolor{blue}{\textbf{+3.6}} & \textcolor{blue}{\textbf{+2.8}} & \textcolor{gray}{\textbf{+0.5}}
 & \textcolor{blue}{\textbf{+1.5}} & \textcolor{blue}{\textbf{+2.2}} & \textcolor{blue}{\textbf{+3.1}} & \textcolor{blue}{\textbf{+1.6}} & \textcolor{gray}{\textbf{+1.7}}
 & \textcolor{blue}{\textbf{+4.3}} & \textcolor{blue}{\textbf{+4.5}} & \textcolor{blue}{\textbf{+3.5}} & \textcolor{blue}{\textbf{3.5}} & \textcolor{gray}{\textbf{+1.3}}\\
\hline
\end{tabular}}
\label{tab:all}
\end{table}

Compared to existing long-tail methods (Table \ref{tab:all}), 
our approach demonstrates superior per-class performance across all datasets and backbones. Existing methods such as CB~\citep{cui2019class} struggle with locating transitions; our method leverages constraints to detect transition boundaries and has substantial improvements in F1 scores. For example, we surpass the second best model LA~\cite{menon2020long} on F1 score by 8.3\%, 3.5\%, and 3.0\% for Breakfast, 50Salads, and Assembly101 respectively for MSTCN backbone. Additionally, our approach has strong frame-wise accuracy because it dynamically adjusts the learning focus based on action and transition learning states. Competing methods such as CB \citep{cui2019class} employ class-wise reweighting without considering the learning state. Focal loss \cite{lin2017focal} overemphasises frames at transition boundaries, even though these are ambiguous~\citep{ding2022temporal, liu2023diffusion}. Due to space constraints, we present the global performance of F1@25 score. Other global results can be found in Supplementary. The results demonstrate our method's ability to balance per-class and global performance. 

Table \ref{tab:group} compares head versus tail group performance. 
Our methods' emphasis on transitions allows us to improve segment-wise performance for both head and tail classes. From a frame-wise perspective, our approach boosts tail classes without compromising the head classes. Notably, Focal loss~\cite{lin2017focal} predominantly focuses on hard boundary frames from head classes due to their high frequency, thereby primarily improving head rather than tail classes.

\begin{table}[t]
\caption{Group-wise result summary across datasets and backbones.}
\centering
\resizebox{1.0\columnwidth}{!}{
\begin{tabular}{c|cc|cc|cc|cc|cc|cc}
\hline
\multirow{2}{*}{\centering{\textbf{Model}}} & \multicolumn{4}{c|}{\textbf{Breakfast}} & \multicolumn{4}{c|}{\textbf{50salads}} & \multicolumn{4}{c}{\textbf{Assembly101}} \\ 
\cline{2-13} & \multicolumn{2}{c|}{\textbf{Accuracy}} & \multicolumn{2}{c|}{\textbf{F1@25}} & \multicolumn{2}{c|}{\textbf{Accuracy}} & \multicolumn{2}{c|}{\textbf{F1@25}} & \multicolumn{2}{c|}{\textbf{Accuracy}} & \multicolumn{2}{c}{\textbf{F1@25}}\\
\cline{2-13} & Head & Tail & Head & Tail & Head & Tail & Head & Tail & Head & Tail & Head & Tail \\ 
\hline
\textbf{MSTCN} & 65.1 & 37.7 & 53.3 & 38.7 & 87.7 & 70.0 & 85.7 & 72.1 & 33.9 & 4.7 & 26.3 & 3.9 \\
+ CB~\cite{cui2019class} & 64.1 & 39.3 & 54.1 & 39.4 & \textbf{88.4} & 69.3 & 85.3 & 72.0 & 34.8 & 6.8 & 28.1 & 6.0 \\
+ LA~\cite{menon2020long} & 64.4 & 40.6 & 56.0 & 38.7 & 87.5 & 69.6 & 86.0 & 71.0 & 34.3 & 6.4 & 27.1 & 5.8 \\
+ Focal~\cite{lin2017focal} & \textbf{66.1} & 36.1 & 53.6 & 38.0 & 88.3 & 70.3 & 84.8 & 73.3 & \textbf{35.3} & 6.6 & 26.3 & 6.4 \\
+ $\tau$-norm~\cite{kang2019decoupling} & 65.3 & 36.2 & 52.7 & 37.4 & 87.6 & 70.3 & 85.1 & 71.6 & 34.0 & 4.3 & 25.9 & 4.2 \\
+ ours(S-NCM) & 65.3 & \textbf{44.0} & \textbf{64.5} & \textbf{44.6} & 87.8 & \textbf{72.5} & \textbf{87.7} & \textbf{75.7} & 34.1 & \textbf{8.7} & \textbf{31.7} & \textbf{7.6} 	\\
\hline
\hline
\textbf{ASFormer} & 69.7 & 39.8 & 69.9 & 43.9 & 90.6 & 77.4 & 87.5 & 80.3 & 35.2 & 5.7 & 29.0 & 4.8 \\
+ CB~\cite{cui2019class} & \textbf{70.2} & 40.8 & 71.2 & 44.7 & \textbf{90.9} & 78.0 & 88.4 & 81.4 & 35.4 & 5.9 & 26.5 & 5.2 \\
+ LA~\cite{menon2020long} & \textbf{70.2} & 40.4 & 71.2 & 44.7 & 90.3 & 78.1 & 89.0 & 81.1 & 36.1 & 5.9 & 27.5 & 5.7 \\
+ Focal~\cite{lin2017focal} & 69.9 & 39.1 & 71.3 & 44.9 & 89.7 & 77.3 & 88.1 & 81.8 & \textbf{36.2} & 7.9 & 29.4 & 7.9 \\
+ $\tau$-norm~\cite{kang2019decoupling} & 69.6 & 41.3 & 69.9 & 44.3 & 90.4 & 77.5 & 87.7 & 80.2 & 32.2 & 4.9 & 21.8 & 3.2 \\
+ ours(S-NCM) & 69.7 & \textbf{44.7} & \textbf{72.5} & \textbf{47.5} & \textbf{90.9} & \textbf{79.5} & \textbf{90.1} & \textbf{82.6} & 35.7 & \textbf{10.9} & \textbf{33.9} & \textbf{10.5} \\
\hline
\end{tabular}}
\label{tab:group}
\end{table}

We further evaluate our method with the SOTA DiffAct~\citep{liu2023diffusion} backbone on Breakfast. Results in Table \ref{tab:diff} demonstrate the effectiveness of our method on improving per class performance, particularly on tail actions, without sacrificing the global performance.

\begin{table}[ht]
\caption{Diffusion backbone trained on Breakfast. Left - per class \& global performance; Right - group-wise performance. }
\centering
\scalebox{0.8}{
 \begin{tabular}{c|ccc|c|c}
 \hline
 \multirow{2}{*}{\centering{\textbf{Model}}} & \multicolumn{4}{c|}{\textbf{Per class}} & \multirow{2}{*}{\textcolor{gray}{\textbf{G\_F1}}} \\
 \cline{2-5} & \multicolumn{3}{c|}{F1@\{10,25,50\}} & Acc. & \\ 
 \hline
 \textbf{DiffAct}~\citep{liu2023diffusion} & 63.3 & 61.5 & 53.6 & 56.2 & \textcolor{gray}{75.5} \\
 + CB~\cite{cui2019class} & \textcolor{red}{-0.4} & \textcolor{red}{-0.5} & \textcolor{red}{-0.4} & 0.0 & \textcolor{gray}{+0.3} \\
 + LA~\cite{menon2020long} & \textcolor{red}{-0.2} & \textcolor{red}{-0.4} & \textcolor{red}{-0.1} & \textcolor{blue}{+0.9} & \textcolor{gray}{+0.1} \\
 + ours(S-NCM) & \textcolor{blue}{\textbf{+0.9}} & \textcolor{blue}{\textbf{+0.5}} & \textcolor{blue}{\textbf{+0.4}} & \textcolor{blue}{\textbf{+1.8}} & \textcolor{gray}{\textbf{+0.4}} \\
 \hline
 \end{tabular}\quad %
 \begin{tabular}{cc|cc}
 \hline
 \multicolumn{2}{c|}{\textbf{Accuracy}} & \multicolumn{2}{c}{\textbf{F1@25}} \\
 \hline
 Head & Tail & Head & Tail \\ 
 \hline
 74.9 & 42.8 & 77.6 & 49.9 \\
 74.1 & 43.3 & 77.5 & 49.1 \\
 \textbf{75.4} & 44.5 & \textbf{78.1} & 49.0 \\
 74.5 & \textbf{46.4} & 77.6 & \textbf{50.8} \\
 \hline
 \end{tabular}
}
\label{tab:diff}
\end{table}
\subsection{Ablation Studies and Analysis}
\noindent \textbf{Components. } 
We evaluate the contributions of the objective function (for class-level bias) and transitional constraints (for transition-level bias) in our proposed constraint optimization Eq. (\ref{eq: op0}), alongside segment-wise post-processing S-NCM in Table \ref{tab:comp}. Incrementally incorporating the objective function and transitional constraints progressively enhances performance. When applied to the naive frame-wise NCM, despite significantly improving per-class accuracy, we observe a noticeable decline in the global F1 score due to over-segmentation. Our proposed S-NCM effectively tackles the over-segmentation issue, thus outperforming the naive NCM. Notably, even without the S-NCM component, our proposed cost-sensitive learning framework substantially enhances per-class performance while maintaining global performance. The inclusion of S-NCM further boosts the final results.

\begin{table}[htb]
\caption{Impact of components on AsFormer. First row is the baseline with no components.}
\centering
\resizebox{0.98\columnwidth}{!}{
\begin{tabular}{cccc|ccc|c|c|ccc|c|c}
\hline
\multirow{3}{*}{\textbf{Objective}} & \multirow{3}{*}{\textbf{Constraint}} & \multirow{3}{*}{\textbf{NCM}} & \multirow{3}{*}{\textbf{S-NCM}} & \multicolumn{5}{c|}{\textbf{Breakfast}} & \multicolumn{5}{c}{\textbf{Assembly101}} \\
\cline{5-14} & & & & \multicolumn{4}{c|}{\textbf{Per class}
} & \multirow{2}{*}{\textcolor{gray}{\textbf{G\_F1}}} & \multicolumn{4}{c|}{\textbf{Per class}} & \multirow{2}{*}{\textcolor{gray}{\textbf{G\_F1}}} \\ 
\cline{5-8} \cline{10-13} & & & & \multicolumn{3}{c|}{F1@\{10,25,50\}} & Acc. & & \multicolumn{3}{c|}{F1@\{10,25,50\}} & Acc. & \\ 
\hline
\xmark & \xmark & \xmark & \xmark & 57.9 & 54.7 & 45.4 & 52.3 & \textcolor{gray}{69.9} & 9.2 & 7.6 & 5.2 & 9.2 & \textcolor{gray}{30.2} \\
\hline
\cmark & \xmark & \xmark & \xmark & 58.4 & 55.6 & 46.4 & 52.9 & \textcolor{gray}{69.8} & 11.2 & 8.9 & 5.8 & 10.7 & \textcolor{gray}{29.8}\\
\cmark & \cmark & \xmark & \xmark & 60.6 & 57.3 & 48.5 & 54.4 & \textcolor{gray}{\textbf{70.8}} & 12.4 & 11.3 & 7.7 & \textbf{12.1} & \textcolor{gray}{29.9}\\
\hline
\cmark & \cmark & \cmark & \xmark & 59.7 & 56.4 & 46.7 & 55.0 & \textcolor{gray}{67.9} & 11.0 & 9.3 & 6.3 & 11.3 & \textcolor{gray}{23.1}\\
\cmark & \cmark & \xmark & \cmark & \textbf{61.0} & \textbf{57.9} & \textbf{49.0} & \textbf{55.1} & \textcolor{gray}{70.3} & \textbf{13.5} & \textbf{12.1} & \textbf{8.7} & \textbf{12.1} & \textcolor{gray}{\textbf{31.7}}\\
\hline
\end{tabular}}
\label{tab:comp}
\end{table}

\noindent \textbf{Impact of $\tau$. }
Table \ref{tab:tau} shows the impact of the hyperparameter $\tau$, applied to gain tensor $G$. This hyperparameter determines the learning balance between global and per-class performance. A smaller $\tau$ results in smoother reweighting factors, prioritising global performance but potentially obscuring the tail actions. Conversely, a large $\tau$ emphasises the learning of tail class, favouring per-class performance. A $\tau$ set too large may overemphasise tail actions, leading to performance drops. Compared to Breakfast, Assembly101 exhibits a larger scale and greater imbalance. The hyperparameter $\tau$ for Assembly101 is then set to a smaller value. More details regarding the selected $\tau$ can be found in Supplementary.
 
\begin{table}[htb]
\caption{Impact of threshold $\tau$ with Asformer.}
\centering
\resizebox{0.85\columnwidth}{!}{
\begin{tabular}{c|cc|ccc|cc|ccc}
\hline
\multirow{2}{*}{$\tau$} & \multicolumn{5}{c|}{\textbf{Breakfast}} & \multicolumn{5}{c}{\textbf{Assembly101}} \\ 
\cline{2-11} & \multicolumn{2}{c|}{\textbf{Per class}} & \multicolumn{3}{c|}{\textbf{Global}} & \multicolumn{2}{c|}{\textbf{Per class}} & \multicolumn{3}{c}{\textbf{Global}} \\ 
\cline{2-11} & F1@25 & Acc. & F1@25 & Acc. & Edit & F1@25 & Acc. & F1@25 & Acc. & Edit \\ 
\hline
0.1 & 56.6 & 53.4 & 69.9 & 71.8 & \textbf{74.0} & \textbf{12.1} & \textbf{12.7} & \textbf{31.7} & \textbf{40.8} & \textbf{32.9} \\
0.3 & \textbf{57.9} & \textbf{55.1} & \textbf{70.3} & \textbf{72.1} & 73.4 & 10.6 & 11.5 & 30.9 & 40.3 & 31.5 \\
0.5 & 57.4 & 54.9 & 69.9 & 71.8 & 72.4 & 10.8 & 11.5 & 30.1 & 39.0 & 31.1 \\
0.7 & 56.8 & 54.6 & 67.3 & 71.1 & 69.7 & 10.6 & 12.2 & 28.3 & 37.1 & 29.9 \\
\hline
\end{tabular}}
\label{tab:tau}
\end{table}

\begin{wrapfigure}{r}{0.45\textwidth}
 \vspace{-7mm}
 \begin{center}
 \includegraphics[width=0.45\textwidth]{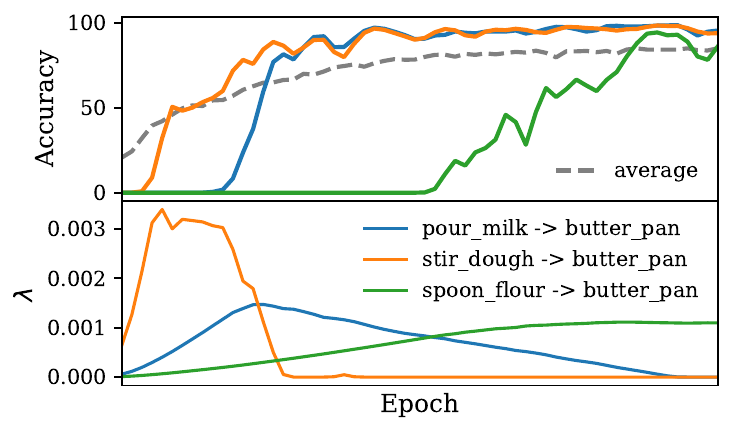}
 \end{center}
 \caption{ 
Transition accuracy and Lagrangian multiplier, $\lambda$, curves during training using AsFormer on Breakfast.
}
 \label{fig:lambda}
\end{wrapfigure}

\noindent \textbf{Lagrangian multiplier $\lambda$.}
We illustrate various types of evolution of accuracy and Lagrangian multipliers, $\lambda$, using transitions related to the action \emph{'butter\_pan'} in Fig.~\ref{fig:lambda}. Our constraints penalise transitions with learning speed slower than the average. An increasing multiplier indicates a violation of its corresponding constraint. In this plot, we observe that transitions to \emph{'butter\_pan'} from other actions exhibit varying learning states. For example, the transition from \emph{'stir\_dough'}, exhibits faster learning speed, with its accuracy surpassing the average accuracy, leading to its multiplier decreasing to zero. Conversely, for less frequent transitions, such as transition from \emph{'spoon\_flour'}, the corresponding $\lambda$ keeps increasing, indicating that its learning state remains below average, prompting more attention towards this transition.

\vspace{3mm}
\noindent \textbf{Computational cost. } Our method requires additional computation costs for calculating the confusion tensor. We estimate the confusion tensor at every epoch for the full training set, which leads to a 30\% longer training time. To mitigate this overhead, we could consider sampling a subset of the training set or employ an exponential moving average approach. 
The testing time complexity remains unaffected compared to the baseline.

\vspace{-1mm}
\section{Conclusion}
We propose a constrained optimization approach to address the bi-level learning bias in temporal action segmentation. The optimization includes an objective to reduce class-level bias arising from class imbalance, and extra transition constraints to reduce transition-level bias stemming from variations in transitions. The problem is transformed into a new cost-sensitive loss function with adaptively adjusted loss weights. Experiments on challenging benchmarks demonstrate the effectiveness of our approach.

\clearpage
\section*{Acknowledgements}
This research is supported by the National Research Foundation, Singapore under its NRF Fellowship for AI (NRF-NRFFAI1-2019-0001). Any opinions, findings and conclusions or recommendations expressed in this material are those of the author(s) and do not reflect the views of National Research Foundation, Singapore.

\bibliography{main}
\clearpage
\section*{Supplementary}
\label{sec:suppl}

\subsection*{A. Long-tail Problem in Temporal Action Segmentation}
Temporal action segmentation methods~\citep{farha2019ms, liu2023diffusion, yi2021asformer} often ignore the long-tail problem, leading to poor performance on tail classes. For instance, state-of-the-art models like MSTCN~\citep{farha2019ms},  ASFormer~\citep{yi2021asformer}, and DiffAct~\citep{liu2023diffusion} fail to predict tail classes accurately. On Breakfast dataset, MSTCN and ASFormer each have zero accuracy for 5 out of 48 classes, while DiffAct misses 4 classes entirely. On Assembly101 dataet, there are 30 classes do not appear in test set. Except those non-appeared classes, MSTCN and ASFormer achieve zero accuracy for 106 and 128 classes of 141 tail classes respectively. Details can be seen in Fig. \ref{fig:underfit_app}.

\begin{figure}[H]
\begin{tabular}{c}
\includegraphics[width=12.2cm]{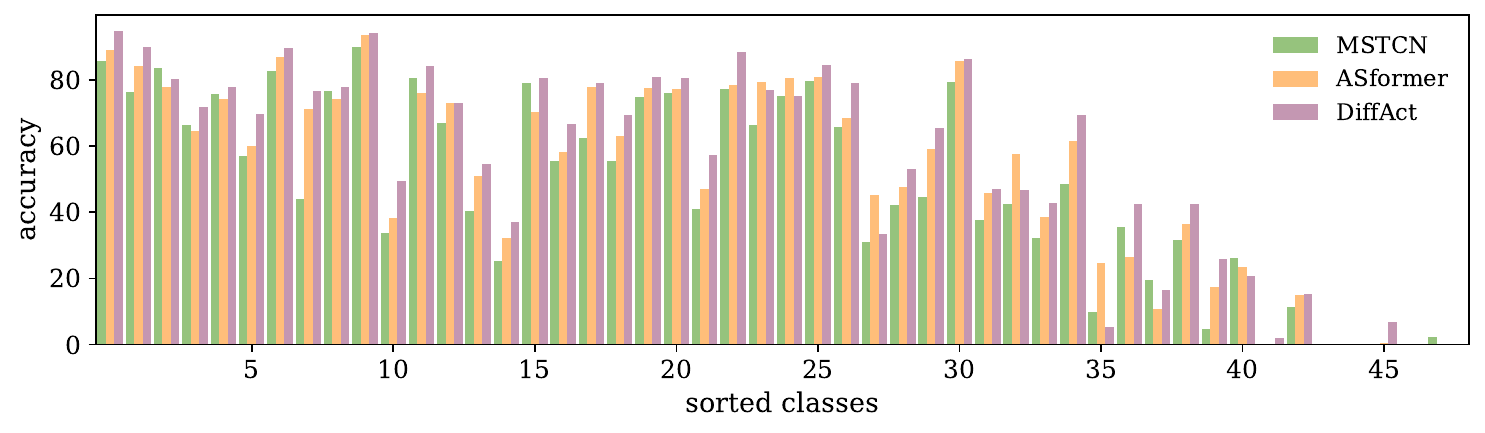}\\
(a) Breakfast \\
\includegraphics[width=12.5cm]{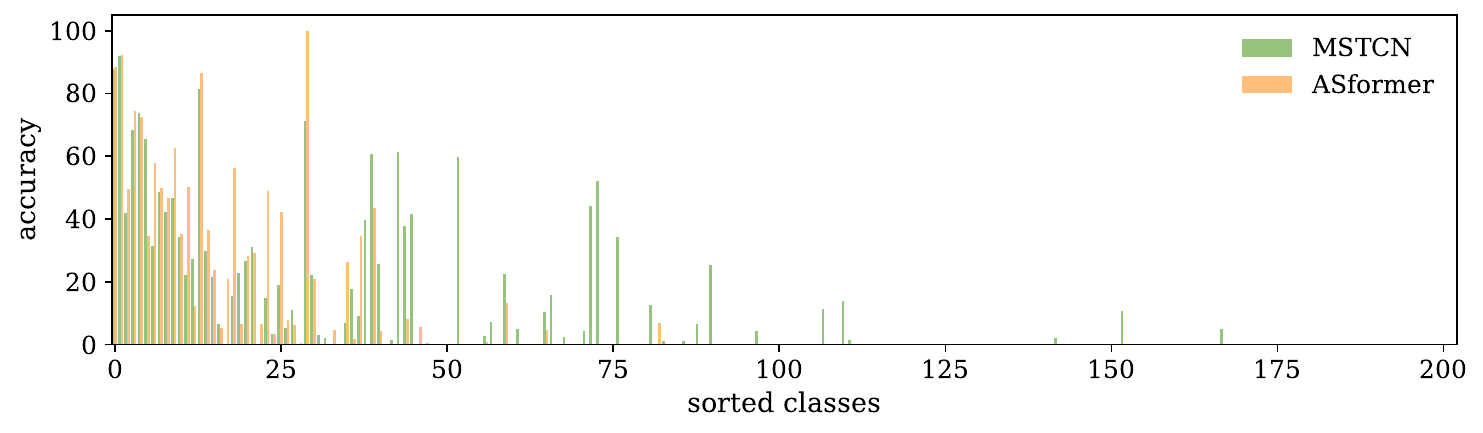} \\
(b) Assembly101
\end{tabular}
\vspace{-0.1cm}
\caption{Class-wise accuracy distribution. }
\label{fig:underfit_app}
\vspace{-0.2cm}
\end{figure}

\subsection*{B. Convert Optimization to Weighted Cross-entropy}
\begin{proposition}
Given a timestamp $t$ and its previous action $u_t$, the optimal classifier of $$\max\limits_{f} \sum_{i, j, k} G_{i,j,k} C_{i,j,k}[f]$$ for a gain matrix $G \in \mathbb{R}^{L \times L \times L+1}$ and the $t^{\text{th}}$ frame takes the form:
$$f^\ast(X, u_t) \in \arg\max\limits_{j \in [L]} \sum_{i} p_i( X) G_{i,j,u_{t}} $$
where $p_i(X)$ is the estimated conditional probability for class $i$ at the current frame $t$ by classifier $f$. 
\end{proposition}

\begin{proof}
\begin{align*}
      \sum_{i,j, k} G_{i,j,k} C_{i, j, k}[f] &= \mathbb{E}_{(X, y_t, u_t)}[\sum_{i,j, k} G_{i,j,k} \mathbb{1}(y_t=i, \hat{y}_t=j, u_t=k)] \\
      &= \mathbb{E}_{(X, y_t, u_{t})}[\sum_{j} G_{y_t,j,u_{t}} \mathbb{1}(\hat{y}_t=j)] \\
      &= \mathbb{E}_{(X, u_{t})} \mathbb{E}_{(y_t \mid X, u_{t})} [\sum_{j} G_{y_t,j,u_{t}} \mathbb{1}(\hat{y}_t=j)] \\
      &= \mathbb{E}_{(X, u_{t})} [\sum_{i, j} p_i(X)  G_{i,j,u_{t}} \mathbb{1}(\hat{y}_t=j)]
\end{align*}

We use the fact in frame-wise classification where the prediction for $y_t$ does not depend on $u_{t}$. It suffices to maximize the above objective point-wise to compute the Bayes-optimal classifier. To predict for a frame labelled as $y_t$ of given input $X$ and the label for the previous action $u_{t}$, the prediction should maximize the term in the expectation.

$$f^\ast(X, u_t) \in \arg\max_{j \in [L]} \sum_{i} p_i(X)  G_{i,j,u_{t}}$$
where $G_{:,:,u_t }$ is a matrix, denoting a slice of $G$.
\end{proof}

In our case, the gain matrix $G_{:,:,u_t}$ is diagonal. The optimal classifier takes the form 
$$f^\ast(X, u_t) \in \arg\max\limits_{i \in [L]} p_i(X) G_{i,i, u_t} \propto \arg\min\limits_{i \in [L]} -G_{i,i, u_t} \log p_i(X) $$
which is the reweighted cross entropy loss and is calibrated for the diagonal gain matrix.

\subsection*{C. Experimental Setting}
\textbf{Dataset distribution}. 
Extra data distribution of 50salads and Assembly101 is illustrated in Fig. \ref{fig:dist_suppl}.

\begin{figure}[H]
\begin{tabular}{c}
\includegraphics[width=5.7cm]{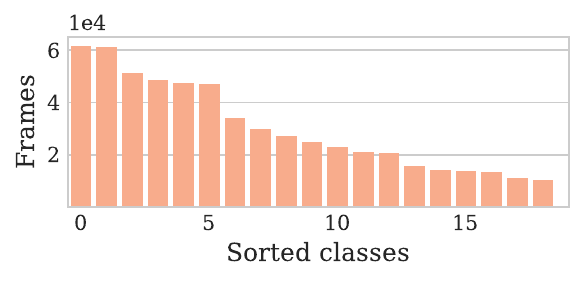}
\includegraphics[width=6.9cm]{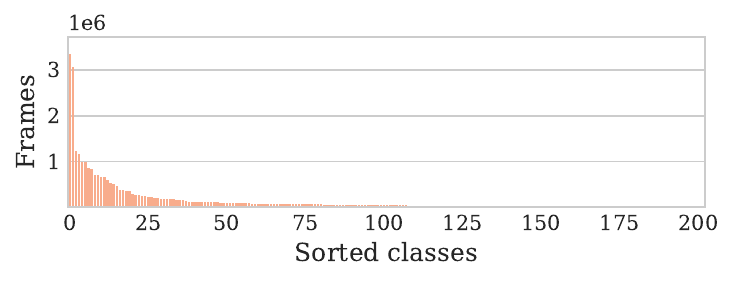} \\
(a) 50salads       \qquad \qquad  \qquad \qquad  \qquad \qquad   (b) Assembly101
\end{tabular}
\vspace{-0.2cm}
\caption{Data distribution of 50salads and Assembly101. }
\label{fig:dist_suppl}
\vspace{-0.2cm}
\end{figure}

\noindent \textbf{Hyperparameters}: 
The used hyperparameters for each dataset, method, and backbone are shown in Table \ref{tab:hyper_suppl}. We omit $\tau$-norm~\cite{kang2019decoupling} as the results always favour $\tau=1.0$ for $\tau$-norm. In our method, we fix the hyperparameter $\epsilon$ in Eq. (\ref{eq: op0}) as 0.9, and the learning rate for multiplier $\gamma$ in Algorithm~\ref{alg:1} as 0.01. 
\begin{table}[htb]
  \caption{Hyperparameters summary}
  \vspace{2mm}
  \label{tab:hyper_break_app}
  \centering
  \resizebox{0.8\columnwidth}{!}{
  \begin{tabular}{cccccc}
    \hline
    \multirow{2}{*}{\textbf{Data}} &  \multirow{2}{*}{\textbf{Model}} & \textbf{Focal}~\citep{lin2017focal} & \textbf{CB}~\citep{cui2019class} & \textbf{LA}~\citep{menon2020long}  &  \textbf{CSL}(ours) \\ 
     & & $\gamma$ & $\beta$ & $\tau$ & $\tau$ \\
    \hline
   \multirow{2}{*}{Breakfast} & MSTCN & 0.5 & 0.9 & 0.5 & 0.5 \\
                              & AsFormer & 1.5 & 0.9 & 0.1 & 0.3 \\
                              & DiffAct & - & 0.99 & 0.3 & 0.7 \\
    \hline
    \multirow{2}{*}{50salads} & MSTCN & 1.5 & 0.9 & 0.5 &  0.7 \\
                              & AsFormer & 0.5 & 0.99 & 0.3 & 0.9 \\
    \hline
    \multirow{2}{*}{Assembly} & MSTCN & 0.5 & 0.9 & 0.1 &  0.3 \\
                              & AsFormer & 0.5 & 0.9 & 0.3 &  0.1 \\
    \hline
  \end{tabular}
  }
\label{tab:hyper_suppl}
\end{table}

\subsection*{D. Additional Results}
\noindent  \textbf{Global performance. }
Evaluation in the main paper primarily focuses on per-class performance, as it better reflects the extent to which the long-tail problem is addressed. Since existing works in temporal action segmentation commonly report global performance, we also present detailed global results across different datasets, backbones, and methods in Table \ref{tab:glb_metric} for completeness. 
Notably, our method, which includes constraints for detecting transitions, demonstrates large improvements on global segment-wise metrics, \ie F1 and edit scores. Although our method may not always lead in frame-wise performance, it still delivers competitive results.
Balancing global and balanced results is challenging due to the trade-off: improving tail often boosts per-class results at the expense of head performance, resulting in the drop in global results. Our method achieves a good trade-off by significantly enhancing per-class performance while still showing competitive results on global metrics. 

\begin{table}[htb]
\caption{Result summary on global metrics.}
\vspace{0.2cm}
\centering
\resizebox{1.0\columnwidth}{!}{
\begin{tabular}{c|ccc|c|c|ccc|c|c|ccc|c|c}
\hline
\multirow{2}{*}{\centering{\textbf{Model}}} & \multicolumn{5}{c|}{\textbf{Breakfast}} & \multicolumn{5}{c|}{\textbf{50salads}} & \multicolumn{5}{c}{\textbf{Assembly101}}  \\ 
\cline{2-16} & \multicolumn{3}{c|}{F1@\{10,25,50\}} & Edit & Acc. & \multicolumn{3}{c|}{F1@\{10,25,50\}} & Edit & Acc. & \multicolumn{3}{c|}{F1@\{10,25,50\}} & Edit & Acc.    \\ 
\hline
\textbf{MSTCN}                            & 63.2 & 57.9 & 46.0 & 66.6 & 67.7   & 78.5 & 75.9 & 67.0 & 71.4 & 81.1 &   30.8 & 27.2 & 20.5 & 30.1 & \textbf{39.8}  \\
+ CB~\cite{cui2019class}                  & 63.6 & 57.9 & 45.7 & 66.8 & 67.4   & 77.7 & 75.5 & 65.8 & 71.1 & 81.0 &   30.0 & 26.7 & 20.2 & 28.4 & 39.7  \\
+ LA~\cite{menon2020long}                 & 63.1 & 57.9 & 45.6 & 67.2 & 67.6   & 78.2 & 75.2 & 66.9 & 70.4 & 80.8 &   29.4 & 26.1 & 20.0 & 29.2 & 39.2  \\
+ Focal~\cite{lin2017focal}               & 63.1 & 57.5 & 45.5 & 67.3 & \textbf{68.5}   & 78.8 & 76.1 & 67.6 & 70.8 & 81.7 &   30.6 & 27.0 & 20.0 & 30.7 & 39.2  \\
+ $\tau$-norm~\cite{kang2019decoupling}   & 62.4 & 57.0 & 45.1 & 66.3 & 67.9   & 77.7 & 75.3 & 66.5 & 70.8 & 81.1 &   31.1 & 27.4 & 20.7 & 30.5 & 39.6  \\
+ ours(S-NCM)                             & \textbf{69.3} & \textbf{64.0} & \textbf{50.9} & \textbf{67.7} & 67.5   & 
                                            \textbf{81.3} & \textbf{79.0} & \textbf{70.2} & \textbf{74.0} & \textbf{81.8} &   
                                            \textbf{32.9} & \textbf{29.5} & \textbf{22.8} & \textbf{30.8} & 39.1  \\
\hline
\hline
\textbf{ASFormer}                         & 75.5 & 69.9 & 56.1 & 74.5 & \textbf{72.4}   & 84.8 & 82.3 & 75.1 & 79.0 & 85.2 &   34.4 & 30.4 & 21.5 & 31.8 & 41.1 \\
+ CB~\cite{cui2019class}                  & 75.6 & 69.7 & 55.8 & 74.9 & 71.9   & 84.9 & 83.2 & 75.7 & 78.7 & 85.8 &   32.6 & 28.2 & 20.1 & 30.6 & 41.0 \\
+ LA~\cite{menon2020long}                 & 75.6 & 69.7 & 56.3 & 74.9 & \textbf{72.4}   & 84.9 & 83.2 & 76.3 & 78.3 & 85.3 &   32.3 & 28.5 & 20.9 & 30.2 & \textbf{41.3} \\
+ Focal~\cite{lin2017focal}               & \textbf{75.7} & \textbf{70.4} & 56.2 & \textbf{75.2} & 72.3   & 85.7 & 83.5 & 75.7 & 79.6 & 84.6 &   34.1 & 30.3 & 22.4 & 32.1 & 41.2 \\
+ $\tau$-norm~\cite{kang2019decoupling}   & 74.9 & 69.1 & 55.7 & 73.6 & 72.2   & 84.7 & 82.2 & 75.2 & 78.9 & 85.2 &   26.4 & 22.7 & 15.9 & 24.3 & 38.5 \\
+ ours(S-NCM)                             & 75.3 & \textbf{70.4} & \textbf{57.5} & 74.3 & 72.1   & \textbf{86.0} & \textbf{84.0} & \textbf{77.8} & \textbf{80.3} & \textbf{86.0} & 
                                            \textbf{34.8} & \textbf{31.7} & \textbf{23.8} & \textbf{32.9} & 40.8 \\
\hline
\end{tabular}}
\label{tab:glb_metric}
\end{table}

\noindent \textbf{Transition detection. } 
The transition constraints help focus on learning hard transitions. Fig. \ref{fig:trans_acc} presents the distribution of transition accuracy, as defined in Eq. (\ref{eq:acc}) on Breakfast test set. Transitions are sorted according to the baseline performance. The results indicate that the model trained under the defined the constraints can detect more transitions, particularly in the tail section where the baseline model struggles to recognise them, demonstrating the efficacy of our transition constraints. Specifically, the baseline Asformer can detect 132 out of 167 transitions, while our cost-sensitive learning(CSL) method successfully detects 11 more transitions. Besides, our method achieves higher average transition accuracy 56.1\% than the baseline 54.3\%.

\begin{figure}[H]
\includegraphics[width=13cm]{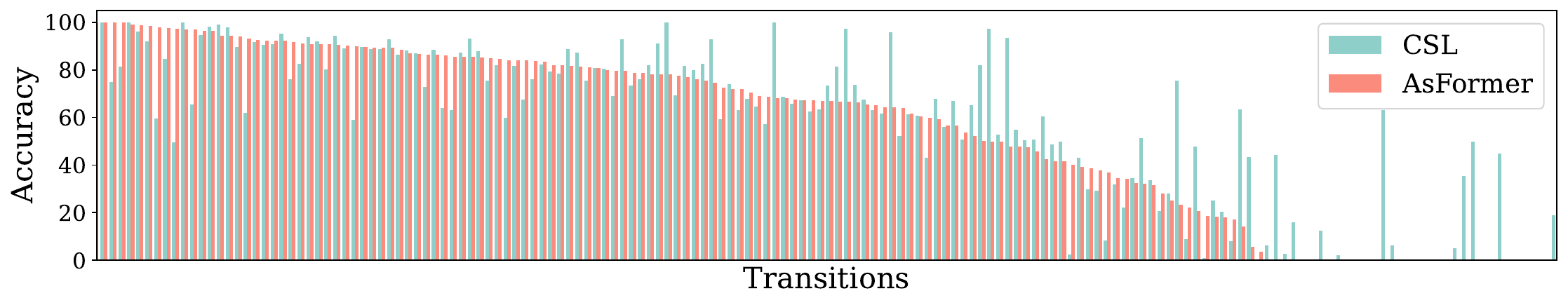}
\vspace{-0.6cm}
\caption{Transition accuracy for AsFormer on Breakfast testset. }
\label{fig:trans_acc}
\vspace{-0.4cm}
\end{figure}

\end{document}